\documentclass{article}

\PassOptionsToPackage{numbers,compress}{natbib}
\usepackage[preprint]{neurips_2026}

\usepackage[utf8]{inputenc}
\usepackage[T1]{fontenc}
\usepackage{hyperref}
\usepackage{url}
\usepackage{booktabs}
\usepackage{amsfonts}
\usepackage{nicefrac}
\usepackage{microtype}
\usepackage{xcolor}
\usepackage{amsmath,amssymb,amsthm,mathtools,bm}
\usepackage{enumitem}
\usepackage{graphicx}
\usepackage{multirow}
\usepackage{array}
\usepackage{subcaption}
\usepackage{tabularx}

\hypersetup{colorlinks=true,citecolor=blue,linkcolor=blue,urlcolor=blue}

\newtheorem{definition}{Definition}
\newtheorem{theorem}{Theorem}
\newtheorem{proposition}{Proposition}
\newtheorem{remark}{Remark}

\newcommand{\E}{\mathbb{E}}
\newcommand{\Pbb}{\mathbb{P}}
\newcommand{\Var}{\mathrm{Var}}
\newcommand{\R}{\mathbb{R}}
\newcommand{\1}{\mathbf{1}}
\newcommand{\cX}{\mathcal{X}}

\title{Who Wins Where? Conformal Model Comparison for Local Superiority}
\author{%
  Yi Zhou \quad Baishi Li \quad Xuan Yao \quad Ke-Wei Huang \\
  Asian Institute of Digital Finance, National University of Singapore \\
  \texttt{zhouyi02@nus.edu.sg} \quad \texttt{baishili@u.nus.edu} \\
  \texttt{yaoxuan@nus.edu.sg} \quad \texttt{dishkw@nus.edu.sg}
}
\date{}

\begin{document}
\maketitle

\begin{abstract}
Standard model comparison is global, aggregating losses across the covariate space to declare a single winner. This can obscure heterogeneous performance, where different models are preferable in different regions. We introduce \emph{conformalized local model comparison}, a split-sample framework for constructing calibrated local best-model maps. Given a model comparison score, such as the difference between two squared losses, the method uses three disjoint splits to fit competing models, estimate local centers and scales from out-of-sample scores, and conformally calibrate residual uncertainty. At a target point, the procedure declares a local winner only when a one-sided conformal bound excludes a tie, with the score's sign determining the favored model. We prove finite-sample marginal control for one-sided erroneous declarations on the realized future comparison score, establish pointwise consistency of the localized mean-score estimator away from tie boundaries, show that aggregate comparison can disagree sharply with the prevalence of local superiority, and derive a squared-loss bias--variance decomposition that clarifies how model structure affects local wins. Synthetic and real-data experiments show that the method recovers heterogeneous winner regions, abstains under uncertainty, and yields higher conditional gain than global selection. Our code is available at \url{https://anonymous.4open.science/r/Model-Winner-submission-DD01}.
\end{abstract}

\section{Introduction}
\label{sec:intro}
Model comparison is central in machine learning. Given two trained predictors, standard practice compares their losses on a validation set, averages those losses across observations, and declares a single global winner. This workflow is simple and often appropriate when one model dominates uniformly. But not all practically important comparisons are global. A structured or theory-guided model may perform better where its inductive bias aligns with the data-generating mechanism, while a flexible black-box model may perform better elsewhere. Likewise, in expert-routing settings, different predictors may be preferable in different subspaces of the covariate space.

This paper studies \emph{conformalized local model comparison}. Rather than asking which model wins on average over the full distribution, we ask which model is superior near a target covariate value $x_0$. Our primary estimand is the conditional mean comparison score $\mu(x)=\E[S(X,Y)\mid X=x]$, such as the difference between two squared losses, where negative values favor model $A$ and positive values favor model $B$. This induces a local best-model map over the feature space. A finite-radius neighborhood average of the mean score provides a stable local smoothing target, and two additional summaries---a local majority notion and a local quantile notion---serve as secondary descriptors of local agreement. The main statistical development, however, focuses on local expected superiority.

A second and the main challenge is inference. Once comparison is made local, the procedure becomes more sensitive to heterogeneity but also more variable. To address this, we use a three-way split. One split fits the competing models, a second split estimates the local score center and scale from out-of-sample comparison scores, and a third split calibrates conformal residuals. A winner is declared only when a one-sided conformal upper bound falls below zero. This construction is deliberately modest: it provides finite-sample \emph{marginal} protection for errors on the \emph{realized future score}, not a finite-sample confidence statement for $\mu(x_0)$ itself.

The theoretical results are correspondingly focused. First, a localized estimator based on an independent evaluation sample recovers the correct sign of $\mu(x_0)$ asymptotically away from tie boundaries. Second, aggregate comparison targets the global mean score and can therefore disagree sharply with the prevalence of local superiority. Third, under squared loss, the local mean score decomposes into squared-bias and variance differences across the covariate space, making the framework especially natural for structured-versus-flexible comparisons. In summary, our contributions are: (i) a local model-comparison formulation based on the conditional mean score $\mu(x)$ and its finite-radius neighborhood average, together with secondary majority and quantile descriptors; (ii) a split-sample, locally centered conformal comparison rule with finite-sample marginal control for false local winner declarations on the realized future score; and (iii) consistency, global-local mismatch, and bias--variance results that characterize when and why local winner regions emerge.

The rest of the paper is organized as follows. Section~2 reviews the literature. Section~3 introduces the local targets. Section~4 presents the conformal procedure. Section~5 develops the theoretical properties. Section~6 summarizes the empirical design and results.

\section{Related Literature}

\noindent\textbf{Conformal prediction, localization, and weighting.}
Conformal prediction provides finite-sample, distribution-free predictive guarantees under exchangeability \cite{vovk2005,shafer2008,lei2018} and is summarized in the overview of \citet{angelopoulos2023}. Subsequent work has made conformal methods adaptive to heterogeneity, for example through conformalized quantile regression \cite{romano2019}, weighted conformal prediction under covariate shift \cite{tibshirani2019}, and localized conformal prediction \cite{guan2023}. These papers motivate the idea that calibration should respond to local feature-dependent structure rather than rely on a global residual distribution. Our paper draws on that intuition, but the inferential object is different: we calibrate a \emph{comparison score between two models}, not a prediction interval for a single model.

\noindent\textbf{Localized model selection and post-selection validity.}
Recent advances in conformal inference have addressed model selection. For instance, localized conformal model selection \cite{wang2026} provides a framework for choosing models that yield efficient conformal intervals while preserving coverage, and \cite{liang2025} studies predictive validity following data-dependent model choices. Our work shares their view that model suitability is heterogeneous, but our objective is complementary. Rather than optimizing interval length or characterizing post-selection predictive coverage, we repurpose conformal calibration as evidence for neighborhood-level superiority under a generic comparison score.

\noindent\textbf{Model comparison and routing.}
The broad motivation also connects to conditional predictive ability testing and local model routing. Global predictive-ability tests compare average losses across evaluation samples \cite{giacomini2006}, while mixture-of-experts and algorithm-selection perspectives emphasize that different models may be preferable in different regions of the input space \cite{jordan1994}. Our paper fits between these viewpoints: it retains the language of loss-based model comparison but replaces one-number global ranking with a local score map over the feature space.

\section{Problem Setup and Local Superiority Targets}

Let $(X,Y)$ be a random pair taking values in $\cX\times\R$, where $X$ denotes covariates and $Y$ denotes the response. We compare two fitted predictive models, $A$ and $B$, with predictors $\hat f_A,\hat f_B:\cX\to\R$. For each model $m\in\{A,B\}$, let $L_m(X,Y)$ denote its loss at $(X,Y)$. To compare the two models, we use a generic score
\begin{equation}
\label{eq:score}
S(X,Y) \coloneqq s\!\bigl(L_A(X,Y),L_B(X,Y);X\bigr),
\end{equation}
where negative scores favor $A$, positive scores favor $B$, and zero denotes a tie. We assume $s(\ell_A,\ell_B;x)$ is weakly increasing in $\ell_A$ and weakly decreasing in $\ell_B$. Examples include the gap $s_{\mathrm{gap}}(L_A,L_B)=L_A-L_B$, the log-ratio $s_{\mathrm{ratio}}(L_A,L_B)=\log((L_A+\tau)/(L_B+\tau))$ with $\tau>0$, and the standardized gap $s_{\mathrm{std}}(L_A,L_B;X)=(L_A-L_B)/v(X)$ with $v(X)>0$.

Our primary population target is the conditional mean score
\begin{equation}
\label{eq:mean}
\mu(x) \coloneqq \mathbb{E}[S(X,Y)\mid X=x].
\end{equation}
When $\mu(x_0)<0$, model $A$ is preferred; when $\mu(x_0)>0$, model $B$ is preferred. To stabilize local estimation, fix a target point $x_0\in\cX$ and a locality parameter $r>0$, and let $K_r:\cX\times\cX\to[0,\infty)$ be a nonnegative neighborhood weight. For the asymptotic theory in Section~5, when $\cX\subseteq\mathbb{R}^d$, we use compactly supported kernel localizers such as the box kernel
\[
K_r(x,x_0)=r^{-d}\1\{\|x-x_0\|\le r\},
\]
which corresponds to the base kernel $K(u)=\1\{\|u\|\le 1\}$, or the Epanechnikov form $K_r(x,x_0)=r^{-d}\bigl(1-\|x-x_0\|^2/r^2\bigr)_+$. In finite-sample implementations, KNN weights are a practical alternative when exact compact support is preferred. The corresponding neighborhood-average score is
\begin{equation}
\label{eq:theta}
\theta(x_0;r)
:=
\frac{\E\!\left[K_r(X,x_0)S(X,Y)\right]}{\E[K_r(X,x_0)]}
=
\frac{\E\!\left[K_r(X,x_0)\mu(X)\right]}{\E[K_r(X,x_0)]}.
\end{equation}
$\theta(x_0;r)$ is a smoothed version of $\mu(x_0)$; under continuity of $\mu$, it approaches $\mu(x_0)$ as $r\to 0$.

\begin{definition}[Local expected winner and best-model map]
Model $A$ is a \emph{local expected winner} at $x_0$ if $\mu(x_0)<0$. For a fixed radius $r$, model $A$ is the \emph{$r$-local expected winner} around $x_0$ if $\theta(x_0;r)<0$. The induced pointwise local best-model map is
\[
\mathcal{W}(x) = A \mathbb{I}_{\{\mu(x) < 0\}} + B \mathbb{I}_{\{\mu(x) > 0\}} + \text{tie} \cdot \mathbb{I}_{\{\mu(x) = 0\}},
\]
and its finite-radius analogue is defined by replacing $\mu(x)$ with $\theta(x;r)$.
\end{definition}

We also record two secondary summaries of local agreement derived from the local score distribution $F_{x_0,r}(t) \coloneqq \E[K_r(X,x_0)\1\{S(X,Y)\le t\}] / \E[K_r(X,x_0)]$ for $t\in\mathbb{R}$. The associated local majority score for model $A$, denoted $\pi_A(x_0;r)$, and the $(1-\beta)$ quantile score $Q_{1-\beta}(x_0;r)$ are
\[
\pi_A(x_0;r) \coloneqq \frac{\E[K_r(X,x_0)\1\{S(X,Y)<0\}]}{\E[K_r(X,x_0)]} \quad \text{and} \quad Q_{1-\beta}(x_0;r) \coloneqq \inf\{t\in\mathbb{R} : F_{x_0,r}(t)\ge 1-\beta\},
\]
for $\beta\in(0,1)$. In our experiments, we define the conditions $\pi_A(x_0;r)>1/2$ and $Q_{1-\beta}(x_0;r)<0$ to be \texttt{local\_def2} and \texttt{local\_def3}, respectively; they indicate that the neighborhood favors model $A$ under local-majority and local-quantile summaries. 
We use \texttt{local\_def1} for the local-mean rule based on the sign of $\theta(x_0;r)$. 
The remainder of the paper focuses on the local mean target $\mu(x)$, equivalently $\theta(x_0;r)$, because it is directly connected to our estimator and theory.


\section{Locally Centered Conformal Comparison}
\label{sec:conformal}

We now attach a finite-sample inferential layer to the comparison score $S(X,Y)$. Let $\{(X_i,Y_i)\}_{i=1}^n$ be an i.i.d. sample, split into three disjoint parts: a fitting split $I_{\mathrm{fit}}$ used only to train the competing predictors $\hat f_A$ and $\hat f_B$, an estimation split $I_{\mathrm{est}}$ used only to estimate a local center and scale from \emph{out-of-sample} scores, and a calibration split $I_{\mathrm{cal}}$ used only for conformal calibration. This separation removes the in-sample-loss issue that would arise if the same observations were used to fit the models and to estimate the local comparison score. Conditional on the models fitted on $I_{\mathrm{fit}}$, the population target $\mu(x)$ and out-of-sample scores $S_j$ for $j \in I_{\mathrm{est}}$ are evaluated according to \eqref{eq:score} and \eqref{eq:mean}. Let $\bar S_{\mathrm{est}}$ and $\widehat V_{\mathrm{est}}$ denote their empirical mean and uncorrected variance; the latter avoids Bessel corrections as the conformal guarantee requires only a strictly positive scale.

For a target point $x$, let $W_x \coloneqq \sum_{j\in I_{\mathrm{est}}} K_r(X_j,x)$. When $W_x > 0$, we define the localized center and regularized scale, for a regularization parameter $\lambda > 0$, as
\[
\widehat\theta_r(x) \coloneqq \sum_{j\in I_{\mathrm{est}}} \frac{K_r(X_j,x)}{W_x} S_j \quad \text{and} \quad \widehat\sigma_r(x) \coloneqq \Biggl[ \lambda + \sum_{j\in I_{\mathrm{est}}} \frac{K_r(X_j,x)}{W_x} \bigl(S_j-\widehat\theta_r(x)\bigr)^2 \Biggr]^{1/2}.
\]
If $W_x = 0$, we fall back to the global estimates $\widehat\theta_r(x) \coloneqq \bar S_{\mathrm{est}}$ and $\widehat\sigma_r(x) \coloneqq (\widehat V_{\mathrm{est}} + \lambda)^{1/2}$. The unstandardized procedure is recovered by setting $\widehat\sigma_r \equiv 1$.

On the calibration split $I_{\mathrm{cal}}$ of size $n_{\mathrm{cal}}$, we compute the scores $S_i$ analogously and form the localized residuals $R_i \coloneqq (S_i-\widehat\theta_r(X_i))/\widehat\sigma_r(X_i)$. Let $R_{(1)} \le \dots \le R_{(n_{\mathrm{cal}})}$ be the sorted residuals. Setting $k_\alpha \coloneqq \lceil(n_{\mathrm{cal}}+1)(1-\alpha)\rceil$, the conformal threshold is $\hat q_{1-\alpha} \coloneqq R_{(k_\alpha)}$ with $\hat q_{1-\alpha} \coloneqq \infty$ if $k_\alpha > n_{\mathrm{cal}}$.

For a test covariate $x_0\in\cX$, the conformal upper bound is $\widehat U_\alpha(x_0) \coloneqq \widehat\theta_r(x_0)+\widehat\sigma_r(x_0)\hat q_{1-\alpha}$. We declare model $A$ the local winner when $\widehat U_\alpha(x_0)<0$. This one-sided rule exclusively controls false declarations in favor of $A$; a symmetric lower-bound declaration for model $B$ follows analogously.

\begin{proposition}[Finite-sample marginal control for the realized score]
\label{prop:conformal_validity}
Let $(X_{n+1},Y_{n+1})$ be an independent test point. Under the assumption of exchangeability conditional on the training and estimation splits, and provided that $\widehat\theta_r$ and $\widehat\sigma_r$ are independent of the test data, the realized score $S_{n+1} \coloneqq s(L_A(X_{n+1},Y_{n+1}), L_B(X_{n+1},Y_{n+1}); X_{n+1})$ satisfies
\[
\mathbb{P}\bigl(S_{n+1}\le \widehat U_\alpha(X_{n+1})\bigr)\ge 1-\alpha \quad \text{and} \quad \mathbb{P}\bigl(\widehat U_\alpha(X_{n+1})<0,\ S_{n+1}\ge 0\bigr)\le \alpha.
\]
\end{proposition}

Proposition~\ref{prop:conformal_validity} is intentionally phrased for the \emph{realized future score}. It does not provide a finite-sample confidence interval or hypothesis test for $\mu(x_0)$ or $\theta(x_0;r)$. Instead, it guarantees that the one-sided conformal rule makes false winner declarations for the realized future score at rate at most $\alpha$ marginally over the test point. Because the conformal guarantee relies exclusively on the exchangeability of the residuals, Proposition~\ref{prop:conformal_validity} remains strictly valid for any functional choice of the comparison score (e.g., gap, log-ratio, or standardized) and any localized centering heuristic, including estimators targeting the majority or quantile summaries defined in Section 3.

For a fixed $x_0$, the same construction yields the one-sided conformal $p$-value $\hat p(x_0) \coloneqq \bigl(1+\sum_{i\in I_{\mathrm{cal}}}\1\{R_i\ge -\widehat\theta_r(x_0)/\widehat\sigma_r(x_0)\}\bigr)/(|I_{\mathrm{cal}}|+1)$, setting $\widehat\sigma_r(x_0)\equiv 1$ if unstandardized. This $p$-value is a slightly less conservative companion to the upper-bound rule $\widehat U_\alpha(x_0)<0$: under continuous residuals the two summaries differ by at most one calibration rank due to empirical-quantile conventions, and ties at the threshold---e.g., for discrete scores---preserve this small discrepancy. 

\section{Theoretical Properties}
\label{sec:theory}
For the first two results, we condition on the predictors produced by the independent fitting split. This isolates the test-time randomness: the evaluation observations remain i.i.d., and the conditional mean $\mu(x)$ from \eqref{eq:mean} is unambiguously defined. Section~\ref{subsec:biasvar} subsequently reintroduces training randomness to interpret an ex-ante version of the expected local score under squared loss.

\subsection{Asymptotic correctness for local expected superiority}

We first study whether a localized estimator recovers the correct sign of $\mu(x_0)$ at a target covariate $x_0$. Using an \emph{independent evaluation sample} and a bandwidth sequence $r_n\to 0$, we define the population-analogue of the estimation-split estimator from Section~4:
\begin{equation}
\label{eq:local-estimator}
\hat\mu_n(x_0) \coloneqq \frac{\sum_{i=1}^n K_{r_n}(X_i,x_0)S_i}{\sum_{i=1}^n K_{r_n}(X_i,x_0)},
\end{equation}
where $S_i$ is the realized score \eqref{eq:score} for the $i$-th observation. Let the true local winner be $\delta^\star(x_0) \coloneqq A$ if $\mu(x_0)<0$ and $B$ if $\mu(x_0)>0$. Its empirical counterpart, $\hat\delta_n(x_0)$, is assigned to $A$ if $\hat\mu_n(x_0)<0$ and to $B$ otherwise. Breaking estimated ties ($\hat\mu_n(x_0)=0$) in favor of $B$ is a convention that differs from the population map $\mathcal{W}$ only on measure-zero events for continuous scores, leaving asymptotic sign recovery unaffected. Furthermore, because sign recovery is inherently unidentifiable on the exact population tie boundary ($\mu(x_0)=0$), the following theorem focuses strictly on the case $\mu(x_0)\neq 0$.

\begin{theorem}[Asymptotic correctness of local winner estimation]
\label{thm:local_winner_consistency}
Fix an interior point $x_0\in\cX$. Assume: (1) $X\in\mathbb{R}^d$ has a density $p_X$ that is continuous at $x_0$ with $p_X(x_0)>0$; (2) the localizer is $K_r(x,x_0)=r^{-d}K((x-x_0)/r)$, where $K$ is bounded, nonnegative, compactly supported, and satisfies $\int_{\mathbb{R}^d}K(u)\,du>0$; (3) $\mu(x)$ is continuous at $x_0$; (4) there exists a neighborhood $U(x_0)$ of $x_0$ such that $\sup_{x\in U(x_0)} \E[S(X,Y)^2\mid X=x] < \infty$; and (5) $r_n\to 0$ and $nr_n^d\to\infty$. Then $\hat\mu_n(x_0)\overset{p}{\longrightarrow}\mu(x_0)$. Consequently, if $\mu(x_0)\neq 0$, then $\mathbb{P}(\hat\delta_n(x_0)=\delta^\star(x_0))\longrightarrow 1$. The same applies to the estimation-split estimator in Section~4 with $n$ replaced by $|I_{\mathrm{est}}|$.
\end{theorem}

\subsection{Aggregate comparison targets the global mean, not local prevalence}

Define the global mean score $\mu_{\mathrm{glob}} \coloneqq \E[S(X,Y)]$ and the aggregate empirical score $\bar S_n \coloneqq \frac{1}{n}\sum_{i=1}^n S_i$. The aggregate selector chooses $\hat\delta_n^{\mathrm{agg}} \coloneqq A$ if $\bar S_n < 0$ and $B$ otherwise. In contrast, the prevalence of local superiority for model $A$ is $\rho_A \coloneqq \mathbb{P}(\mu(X)<0)$, which depends exclusively on the sign of the local conditional mean score defined in \eqref{eq:mean}.

\begin{proposition}[Aggregate selection can disagree with local superiority prevalence]
\label{prop:global_local_mismatch}
Suppose $\E[|S(X,Y)|]<\infty$. Let the true aggregate winner be $\delta^{\mathrm{agg},\star} \coloneqq A$ if $\mu_{\mathrm{glob}} < 0$ and $B$ if $\mu_{\mathrm{glob}} > 0$. If $\mu_{\mathrm{glob}} \neq 0$, then $\mathbb{P}(\hat\delta_n^{\mathrm{agg}}=\delta^{\mathrm{agg},\star})\longrightarrow 1$. Moreover, for any $q\in(1/2,1)$, there exists a joint distribution of $(X,S)$ such that $\mathbb{P}(\mu(X)<0)\ge q$ but $\E[S]>0$. In such cases, $\mathbb{P}(\hat\delta_n^{\mathrm{agg}}=B)\longrightarrow 1$ even though model $A$ is locally superior on at least a $q$-fraction of the covariate distribution.
\end{proposition}

Proposition~\ref{prop:global_local_mismatch} clarifies an elementary but practically profound point routinely obscured in standard model evaluation: the global winner has no mathematical binding to the proportion of local winners. Aggregate averaging evaluates the net magnitude of the global mean score $\mu_{\mathrm{glob}}$, allowing massive losses in a small covariate region to dominate the sum. Conversely, local prevalence depends strictly on the spatial frequency of the expected score's sign. Consequently, global selection can completely reject a model even if it is the strictly superior choice for $99\%$ of the population.

\subsection{Ex-Ante Bias-Variance Interpretation for Squared Loss}
\label{subsec:biasvar}

While the preceding sections condition on fitted predictors to evaluate superiority, we now reintroduce the randomness of the training data to understand \emph{why} heterogeneous regions of superiority emerge. By analyzing an ex-ante expected local score over both the training procedure and the test distribution, we can mechanically trace local wins back to standard learning-theoretic properties.

\begin{proposition}[Squared-loss local bias--variance decomposition]
\label{prop:squared_loss_decomp}
Assume \(Y=f^\star(X)+\varepsilon\) with \(\E[\varepsilon\mid X]=0\), and assume the training sample \(\mathcal D\) used to fit the predictors is independent of the fresh test pair \((X,Y)\). Let the training-random squared-loss gap be $S_{\mathcal{D}}(X,Y) \coloneqq (\hat f_A(X)-Y)^2 - (\hat f_B(X)-Y)^2$, where the fitted predictors are trained on a random training sample $\mathcal{D}$. For $m\in\{A,B\}$, define $\mathrm{Bias}_m(x) \coloneqq \E_{\mathcal{D}}[\hat f_m(x)]-f^\star(x)$ and $\mathrm{Var}_m(x) \coloneqq \Var_{\mathcal{D}}(\hat f_m(x))$, where expectations are taken over $\mathcal{D}$. If $\E_{\mathcal{D}}[\hat f_m(x)^2]<\infty$ and $\E[\varepsilon^2\mid X=x]<\infty$, then the ex-ante conditional mean score $\mu_{\mathrm{ex}}(x) \coloneqq \E_{\mathcal{D},Y\mid X=x}[S_{\mathcal{D}}(X,Y)\mid X=x]$ satisfies
\begin{equation}
\label{eq:conditional-bias-var-main}
\mu_{\mathrm{ex}}(x) = \mathrm{Bias}_A^2(x)-\mathrm{Bias}_B^2(x)+\mathrm{Var}_A(x)-\mathrm{Var}_B(x).
\end{equation}
Equivalently, the ex-ante smoothed neighborhood score $\theta_{\mathrm{ex}}(x_0;r) \coloneqq \E[K_r(X,x_0)\mu_{\mathrm{ex}}(X)]/\E[K_r(X,x_0)]$ satisfies
\begin{equation}
\label{eq:local-bias-var-main}
\theta_{\mathrm{ex}}(x_0;r) = \frac{\E\bigl[K_r(X,x_0)\bigl\{\mathrm{Bias}_A^2(X)-\mathrm{Bias}_B^2(X)+\mathrm{Var}_A(X)-\mathrm{Var}_B(X)\bigr\}\bigr]}{\E[K_r(X,x_0)]}.
\end{equation}
\end{proposition}

Although mathematically elementary, Proposition~\ref{prop:squared_loss_decomp} provides the fundamental explanatory mechanism for local conformal comparison. Because irreducible noise cancels out, the ex-ante local mean score reduces strictly to a competition between squared-bias and variance differences. This decomposition formalizes the spatial trade-off between a structured, high-bias/low-variance model (e.g., OLS) and a flexible, low-bias/high-variance model (e.g., Gradient Boosting). In smooth regions where the data generating mechanism aligns with the structured model's inductive bias, its variance advantage dictates local superiority. Conversely, in highly non-linear regions, the flexible model's localized bias reduction overcomes its higher variance, forcing a flip in the local winner map. We empirically validate this exact spatial mechanism in Section~6.

\section{Experiments}
\label{sec:experiments}

We validate our approach on synthetic and real-world datasets. The synthetic experiments serve three purposes. First, they test whether local comparison can recover heterogeneous winner regions, as formalized in Proposition~\ref{prop:global_local_mismatch}. Second, they test whether conformal calibration converts local evidence into statistically reliable winner declarations, as guaranteed by Proposition~\ref{prop:conformal_validity}. Third, they provide diagnostics for the asymptotic and bias--variance mechanisms in Theorem~\ref{thm:local_winner_consistency} and Proposition~\ref{prop:squared_loss_decomp}. The real-data experiments, reported in Section~\ref{subsec:real_exp_design}, evaluate whether the conformal declarations select regions with positive realized gain when the oracle local winner map is unavailable.

\subsection{Synthetic data experiments}
\label{subsec:exp_design}

We use five one-dimensional synthetic designs. This setting makes local winner regions visually inspectable and allows us to evaluate against an oracle local comparison target. In all designs, $X\sim \mathrm{Unif}[0,1]$ and $Y=f^\star(X)+\varepsilon$. Model $A$ is a gradient-boosted tree, implemented by \texttt{LGBMRegressor} with $500$ estimators, learning rate $0.05$, and depth $6$; model $B$ is ordinary linear regression. Each replicate has $n=5000$ observations, split evenly into fitting, estimation, calibration, and test splits. Results are averaged over $10$ random seeds, with nominal level $\alpha=0.10$ and KNN neighborhood size $k=\lfloor\sqrt{|I_{\mathrm{est}}|}\rfloor\approx 35$. The primary score is the squared-loss gap
\begin{equation}
\label{eq:gap_mse}
S(X,Y) = (\hat f_A(X)-Y)^2 - (\hat f_B(X)-Y)^2.
\end{equation}
where $S<0$ favors the boosted tree. We evaluate local winner using the post-fit oracle $\mu(x)=(\hat f_A(x)-f^\star(x))^2-(\hat f_B(x)-f^\star(x))^2$, which compares the two already-fitted predictors at each covariate value. Full DGP formulas and design motivations for all five synthetic cases are provided in Appendix~\ref{app:DGP}. Evaluation metrics for the synthetic experiments are defined in Appendix~\ref{app:metrics}.

\paragraph{Results \#1: Local comparison recovers winner regions hidden by global averages.}
Figure~\ref{fig:three_winner} illustrates a global--local mismatch scenario, corresponding to Case~3 in Table~\ref{tab:dgps_app}. 
Here, the data-generating function is linear outside $[0.42,0.58]$ but contains a strong nonlinear oscillation inside this interval. 
As a result, model $A$ wins inside the island, while model $B$ wins outside it. 
The baseline \texttt{global\_mean} measures only the aggregate mean score over the full input space and declares a single winner everywhere. 
Because the island has sufficiently large loss differences, \texttt{global\_mean} declares $A$ everywhere and misses most of the region where $B$ is locally better.

To diagnose this structure, we evaluate the three non-conformal local winner maps introduced in Section~3: \texttt{local\_def1}, \texttt{local\_def2}, and \texttt{local\_def3}. 
These rules are used to study local-region recovery. 
In Case~3, \texttt{global\_mean} attains only $0.293$ winner accuracy, compared with $0.823$ for \texttt{local\_def1} and $0.908$ for \texttt{local\_def3}; full results are reported in Table~\ref{tab:winner_metrics_app}.
As shown in Figure~\ref{fig:three_winner}, \texttt{local\_def1} closely follows the oracle $\theta(x_0;r)$, \texttt{local\_def2} is more sensitive to local score noise, and \texttt{local\_def3} is more conservative near the island boundary. 
These local rules reveal winner regions that are invisible to \texttt{global\_mean}, supporting Proposition~\ref{prop:global_local_mismatch}: the aggregate winner and the prevalence of local winners are distinct objects.

\begin{figure}[t]
\centering
\includegraphics[width=0.85\linewidth]{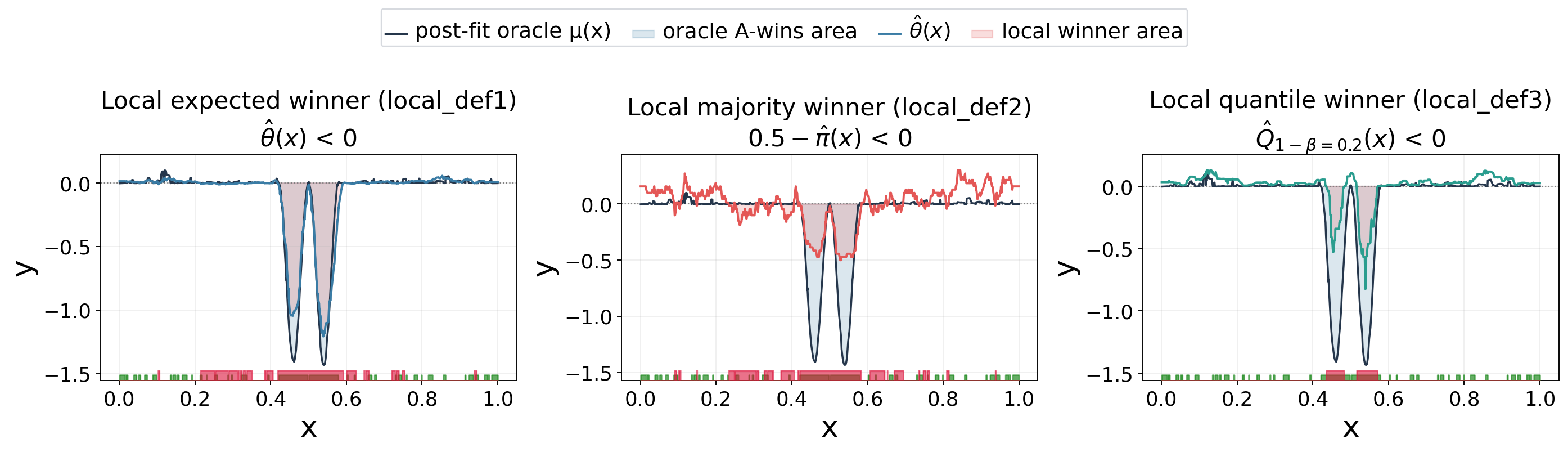}
\caption{Three non-conformal local winner methods on Case~3. Left: local expected winner. Middle: local majority winner. Right: local quantile winner with $\beta=0.20$.}
\label{fig:three_winner}
\end{figure}

\paragraph{Results \#2: Conformal calibration controls false winner declarations.}

We next test whether local winner evidence can be converted into statistically controlled declarations. 
Figure~\ref{fig:validity} plots the false winner rate against the nominal level $\alpha$ across all five synthetic cases. 
The conformal methods follow the one-sided upper-bound rule in Section~4, declaring $A$ only when $\widehat U_\alpha(x)<0$. 
Specifically, \texttt{global\_split\_cp} uses a constant center $\bar S_{\mathrm{est}}$ and unit scale, \texttt{local\_split\_cp} uses the localized center $\widehat\theta_r(x)$ with unit scale, and \texttt{local\_std\_cp} uses both $\widehat\theta_r(x)$ and the localized scale $\widehat\sigma_r(x)$. 
These variants separate the effects of localizing the center and the scale. Across all five cases, the conformal methods stay below the diagonal throughout, consistent with Proposition~\ref{prop:conformal_validity}. 
In contrast, non-conformal plug-in maps can over-declare when realized scores are noisy. 

\begin{figure}[t]
\centering
\includegraphics[width=\linewidth]{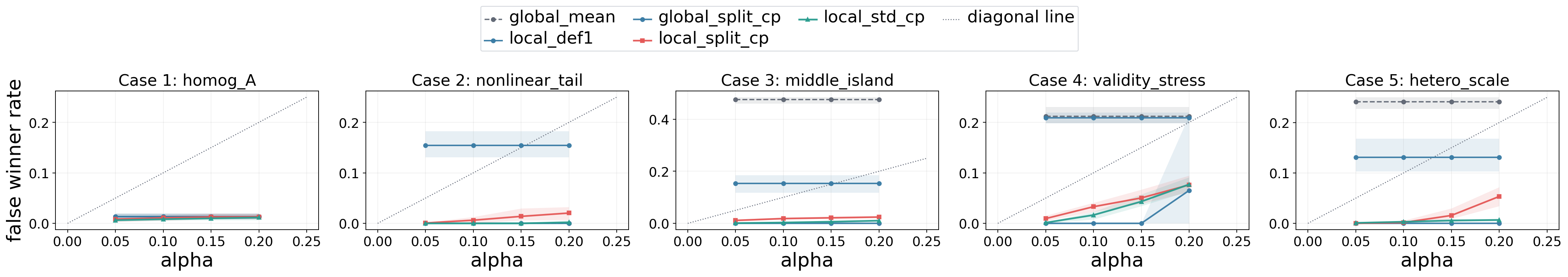}
\caption{False winner rate versus $\alpha$. The diagonal marks the target level. Conformal methods control false winner declarations, while non-conformal methods can over-declare under noisy realized scores.}
\label{fig:validity}
\end{figure}

Figure~\ref{fig:winner_maps} illustrates the distinction between identifying and certifying a local winner in Case~3. The non-conformal local mean rule declares $A$ when the estimated local mean is negative. The conformal rule declares $A$ only when the conformal upper bound is below zero, producing a conservative subset of the estimated nonlinear island. Thus, conformal calibration changes the output from a forced local winner map to a statistically controlled declaration rule.

\begin{figure}[t]
\centering
\includegraphics[width=0.55\linewidth]{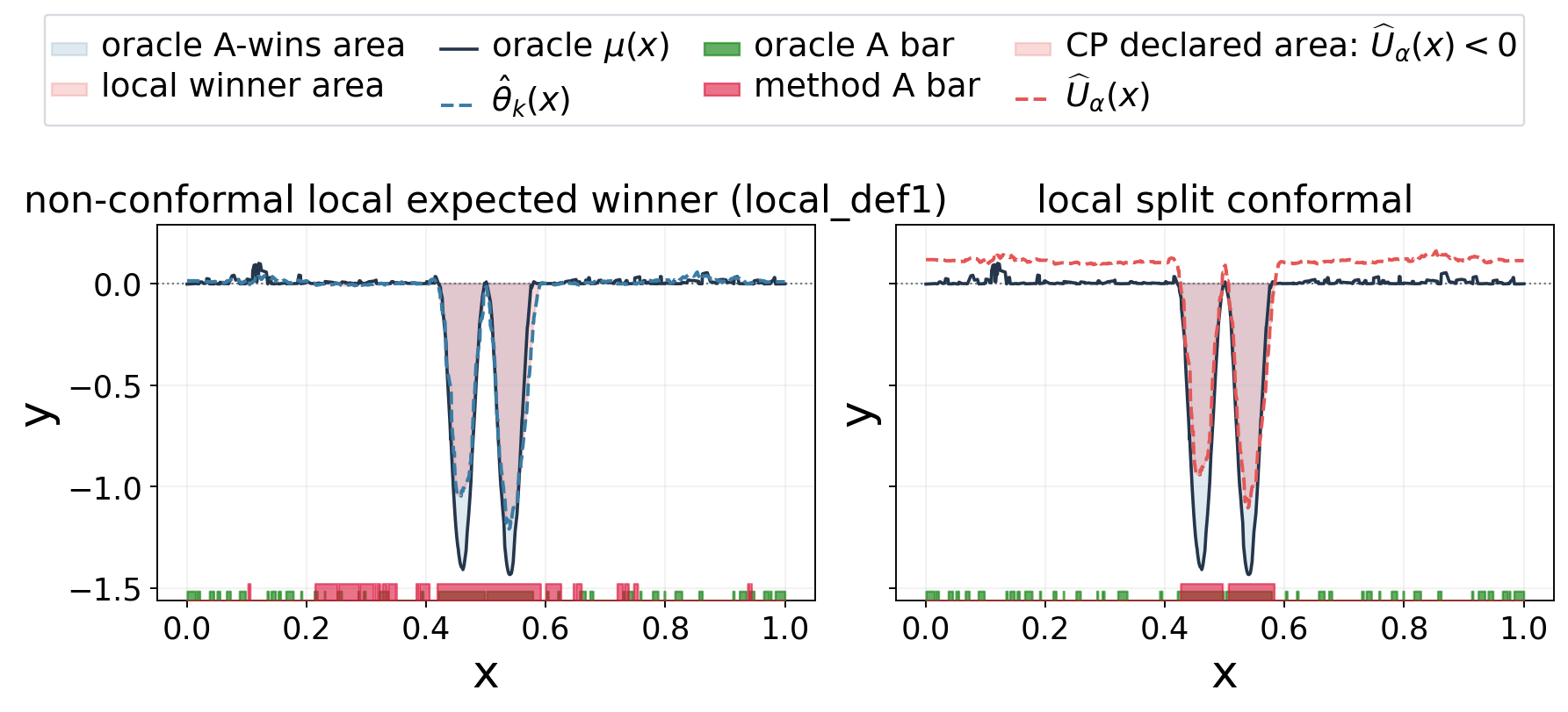}
\caption{Case~3 comparison between non-conformal local winner identification and conformal winner declaration. The conformal rule certifies a conservative subset of the local $A$-winning island.}
\label{fig:winner_maps}
\end{figure}

Among conformal methods, local centering and local scaling improve the efficiency of declarations. A global conformal rule uses one center for the entire covariate space, so it either certifies $A$ broadly or abstains broadly. Local conformal methods can instead select only the regions where $A$ is locally favorable. 
Figures~\ref{fig:selection_power} report selection rate and oracle $A$ power as functions of $\alpha$. Local conformal methods select fewer points than the oracle $A$ region, reflecting the price of finite-sample error control, but their selected points are more spatially targeted than those of the global conformal baseline. 

\begin{figure}[t]
\centering
\includegraphics[width=\linewidth]{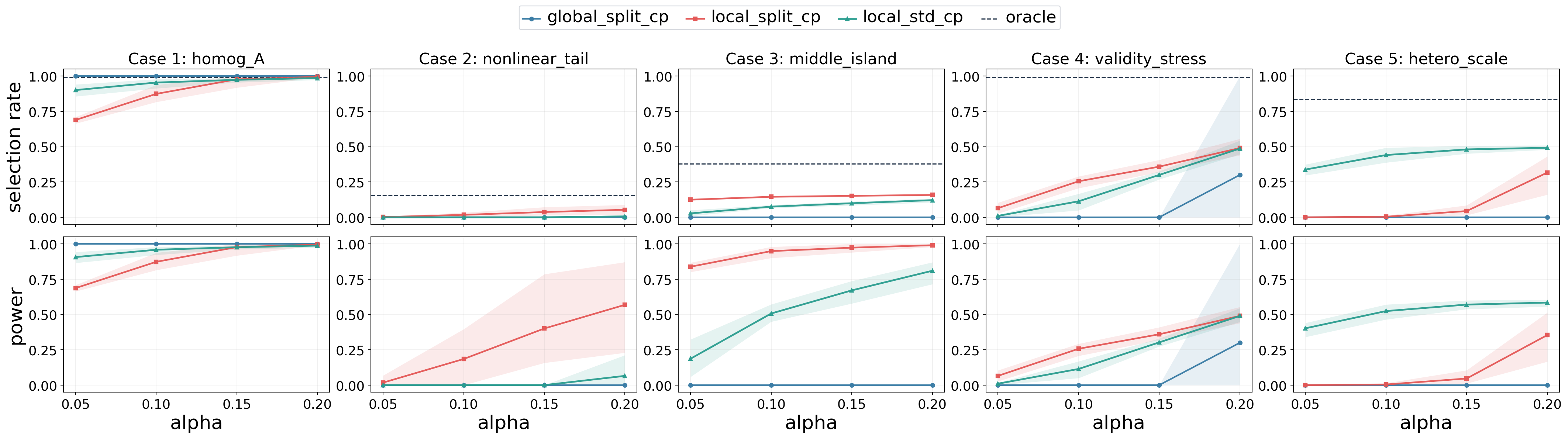}
\caption{Conformal selection versus $\alpha$. Top: selection rate with the dashed line indicating the oracle $A$-winning rate. Bottom: oracle $A$ power, $\mathbb{P}(\hat w_\alpha=A\mid\mu(X)<-\tau)$, for the conformal methods.}
\label{fig:selection_power}
\end{figure}

\paragraph{Results \#3: Diagnostics for the theoretical mechanisms.}
Figure~\ref{fig:biasvar} verifies the squared-loss bias--variance decomposition in Case~3.  Proposition~\ref{prop:squared_loss_decomp} is supported: outside the nonlinear island, the linear model has low variance and small bias, so $B$ wins locally; inside the island, the boosted tree's bias reduction dominates its variance cost, so $A$ wins. 

\begin{figure}[t]
    \centering
    \begin{minipage}{0.55\linewidth}
        \centering
        \includegraphics[width=\linewidth]{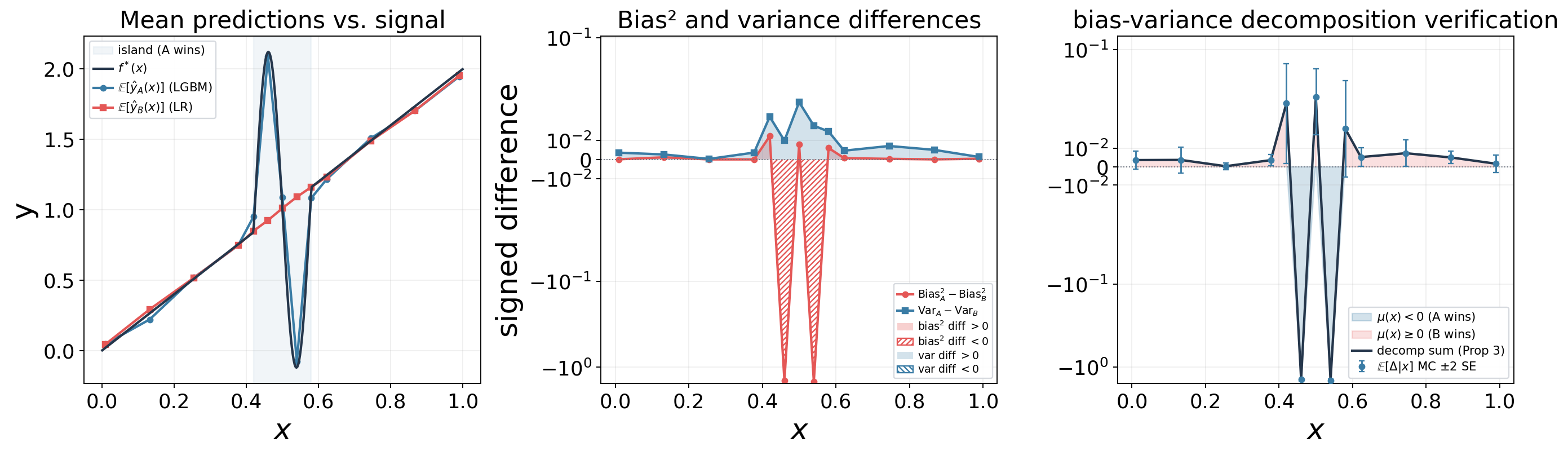}
        \caption{Bias--variance decomposition on Case~3. Left: signal and mean predictions. Middle: signed bias and variance components; negative values favor model~$A$ and positive values favor model~$B$. Right: their sum matches the Monte Carlo loss gap, showing $A$ wins inside the island and $B$ outside.}
        \label{fig:biasvar}
    \end{minipage}\hfill
    \begin{minipage}{0.41\linewidth}
        \centering
        \includegraphics[width=\linewidth,height=0.18\textheight,keepaspectratio]{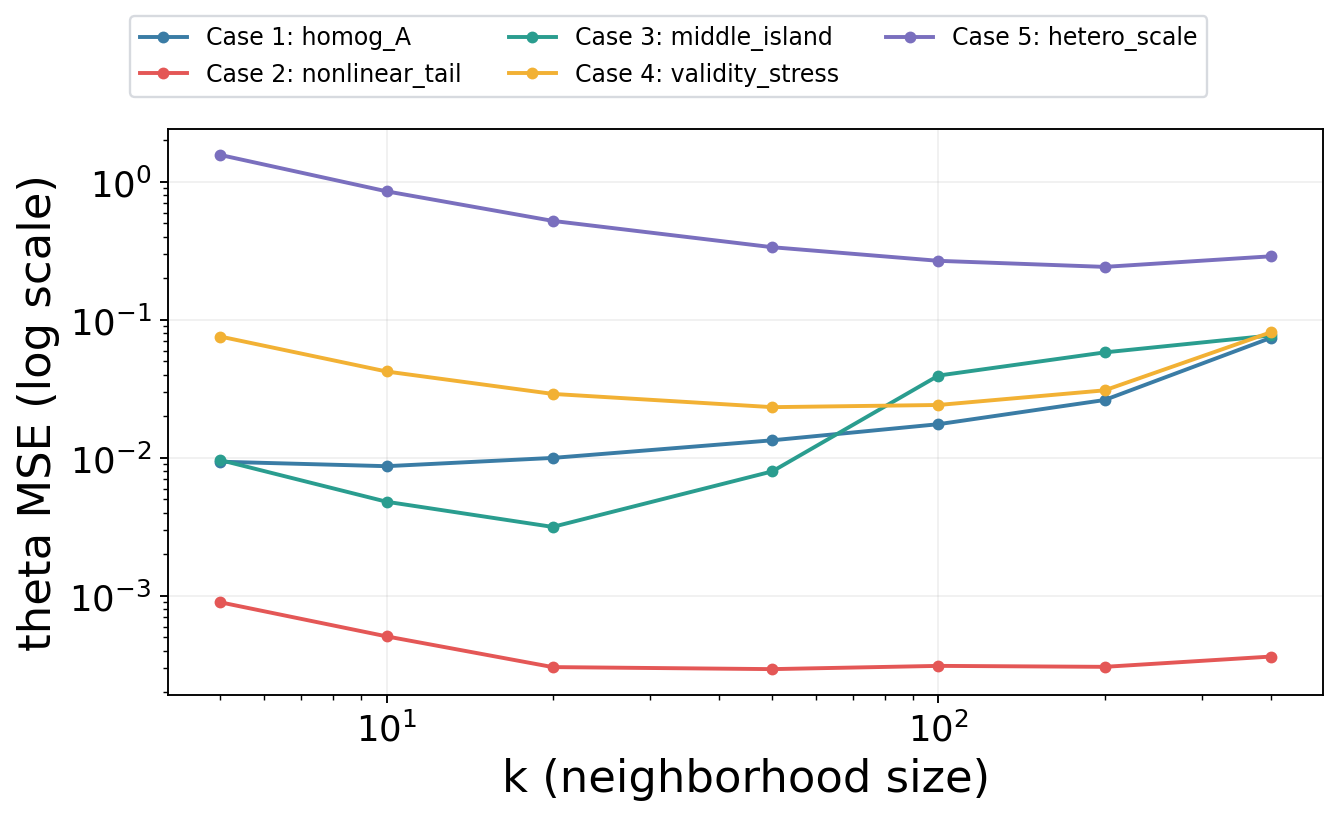}
        \caption{Local MSE versus neighborhood size $k$. Small $k$ gives noisy local estimates. The optimal $k$ is case-dependent.}
        \label{fig:thetamse}
    \end{minipage}
\end{figure}

Figure~\ref{fig:thetamse} studies the locality tradeoff behind Theorem~\ref{thm:local_winner_consistency}. The MSE of the local estimator first decreases and eventually increases when neighborhoods become too wide. Results of these five cases illustrates the optimal $K$ needs to be tuned.


\subsection{Real data experiments}
\label{subsec:real_exp_design}

We evaluate on four public regression benchmarks: \texttt{concrete}, \texttt{auto\_mpg}, \texttt{blog\_data}, and
\texttt{facebook\_1}.\footnote{\url{https://github.com/yromano/cqr}}\footnote{\url{https://archive.ics.uci.edu/dataset/9/auto+mpg}}
Model~$A$ is a small fully connected neural network (MLP) with one hidden layer, and model~$B$ is linear regression. The dataset
statistics are summarized in Table~\ref{tab:realdata_summary} in the appendix. Across the four datasets, the MLP has better global mean performance on average across seeds, while linear regression may still be locally superior in structured subregions. The goal is to find local regions where linear regression outperforms the MLP. Since no oracle map is available in real-data experiments, we report model~$B$'s selection rate and the conditional gain $\mathbb{E}[L_A-L_B\mid \hat w_\alpha=B]$, which measures the realized advantage of routing to model~$B$ in the declared region relative to using model~$A$ everywhere, and visualize the declared regions in a two-dimensional PCA projection.


\paragraph{Result \#1: Model A vs.\ model B winning regions.}
Figure~\ref{fig:real_regions} visualizes winner declarations in a 2D PCA projection, with colors denoting the across-seed frequency of declaring model~$B$. The global mean rule is a global decision, selecting all points as either $A$ or $B$ within each seed; intermediate colors arise only from variation across seeds. Local non-conformal maps (columns 1--2) identify subregions where model $B$ achieves lower loss. Columns 3--4 highlight our method's advantage: the global conformal rule abstains entirely, whereas our local conformal rule successfully certifies $B$-favorable regions. Though smaller than the non-conformal baseline due to rigorous false-winner control (Proposition~\ref{prop:conformal_validity}), our localized approach converts noisy estimates into conservative, statistically controlled subregions where model~$B$ is favored.

\begin{figure}[t]
\centering
\includegraphics[width=0.6\linewidth]{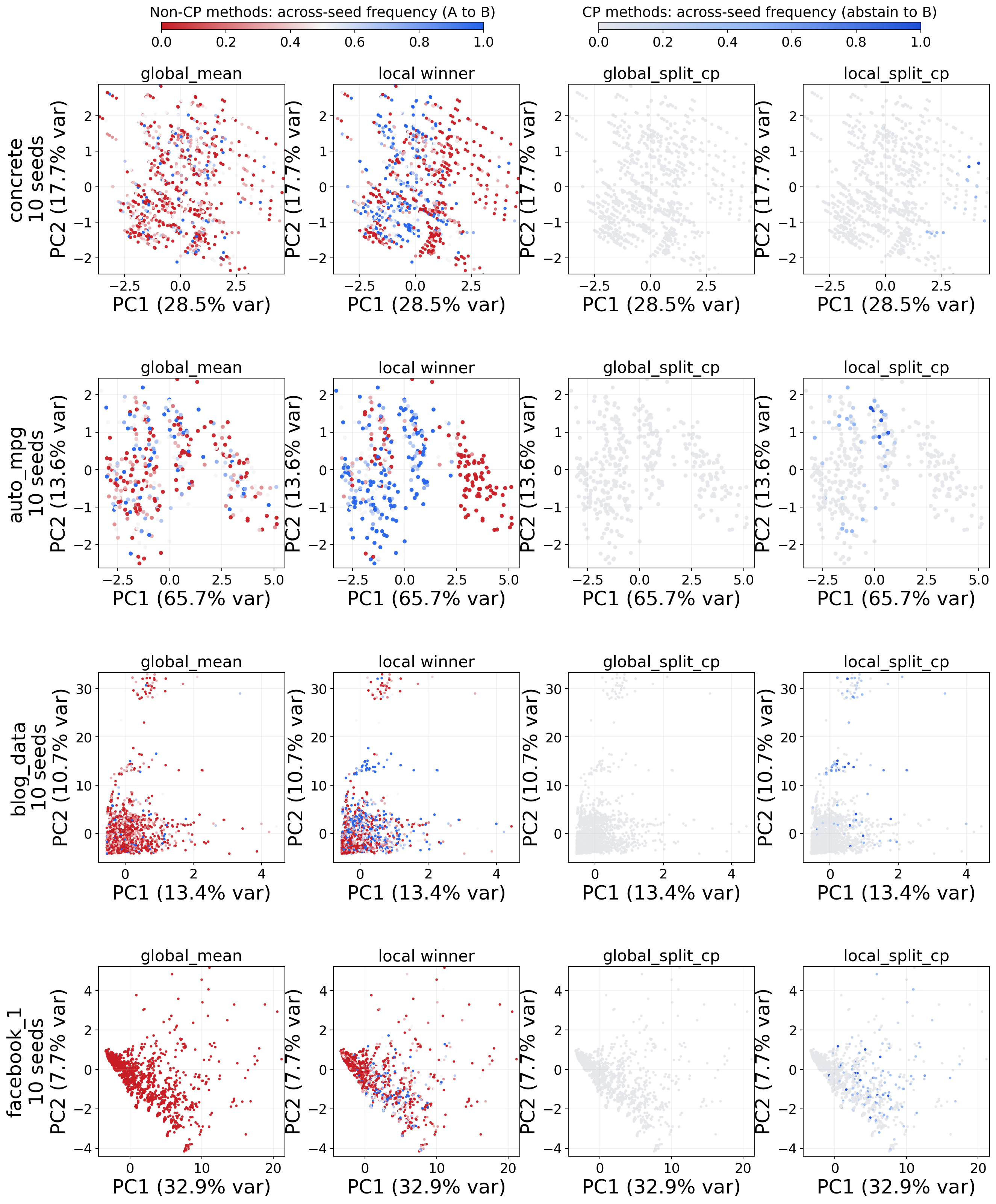}
\caption{A/B winner regions in a 2D PCA projection. Colors show across-seed selection frequency: red-to-blue for non-CP methods ($A$ to $B$), and gray-to-blue for CP methods (abstention to $B$).}
\label{fig:real_regions}
\end{figure}


\paragraph{Result \#2: Selection vs.\ conditional gain.}
Figure~\ref{fig:real_gain} plots conditional gain—the average realized test-set advantage of model~$B$ over model~$A$ within the declared region—against selection rate across varying $\alpha$. Increasing $\alpha$ expands declarations, though typically diluting concentration in high-gain areas. Non-conformal methods achieve higher selection rates but lower conditional gain, reflecting overly broad declarations that lack false-winner control. The global conformal baseline (yellow point) abstains entirely, yielding zero gain. Conversely, the local conformal method (red line) achieves the highest conditional gain, outperforming its non-conformal counterpart (blue point). This confirms its abstention mechanism effectively isolates a smaller, highly targeted region of genuine superiority.

\begin{figure}[t]
\centering
\includegraphics[width=0.9\linewidth]{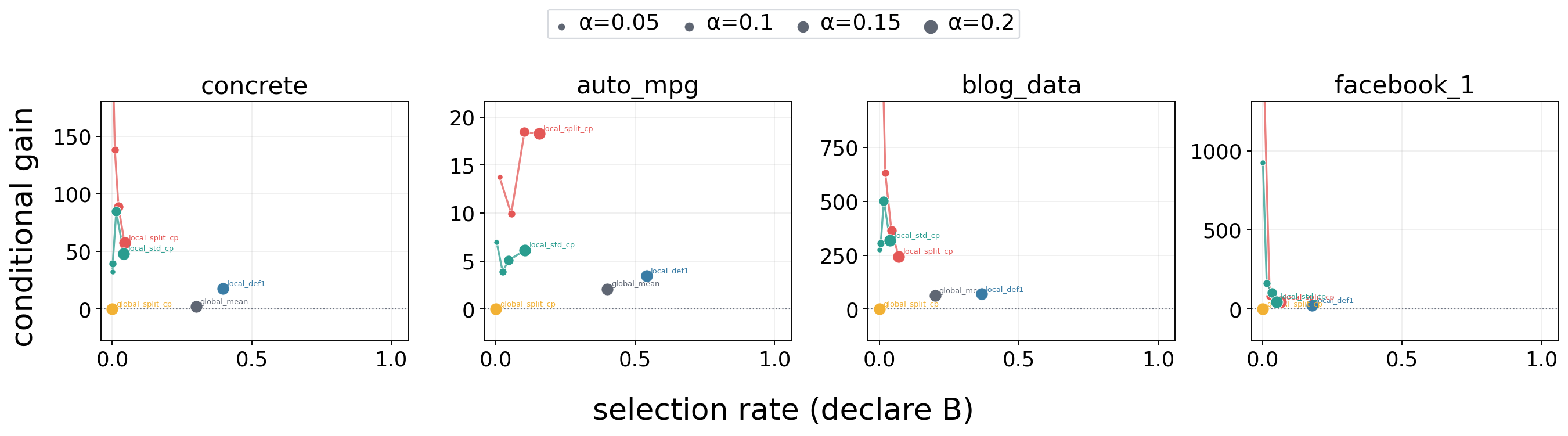}
\caption{Conditional gain versus selection rate on real datasets. Larger points correspond to a larger $\alpha$. 
Y-axis indicates a larger realized gain in the declared region, while x-axis indicates a higher selection rate. 
Local conformal methods select smaller but higher-gain regions.}
\label{fig:real_gain}
\end{figure}

\section{Conclusion}

We presented a novel localized framework for comparing predictive models: the conformalized conditional mean comparison score. To rigorously ground this methodology, we established three core theoretical properties: the asymptotic consistency of local winner estimation, the formal disconnect between global aggregate metrics and local superiority prevalence, and an ex-ante bias-variance decomposition that explicitly explains the mechanical emergence of local advantages. Our empirical experiments validate these theoretical claims, demonstrating the framework's capacity to detect heterogeneous model performance where global metrics fail.


\section{Limitations}
\label{sec:limitations}

Locally centered conformal comparison has several limitations. First, the independent three-way split $(I_{\mathrm{fit}},I_{\mathrm{est}},I_{\mathrm{cal}})$ is clean but sample-inefficient. Cross-fitting or out-of-bag prediction can recycle data in practice, but extending the exact finite-sample guarantee in Proposition~\ref{prop:conformal_validity} to such dependent constructions requires additional analysis. Second, localization in raw covariate space suffers from the curse of dimensionality. In sparse regions, the fallback $\widehat\theta_r(x_0)=\bar S_{\mathrm{est}}$ keeps the procedure well-defined and preserves validity, but sacrifices local adaptation; high-dimensional applications should therefore apply $K_r$ to a learned low-dimensional representation or sufficient summary. Third, the conformal guarantee is marginal for the realized future score, not a finite-sample confidence statement for $\mu(x_0)$ or $\theta(x_0;r)$. Near tie boundaries, where $\mu(x_0)\approx 0$, the method will often abstain rather than declare a winner; this is appropriate but may reduce coverage of weak-signal regions. Finally, constructing a full local best-model map over many target points $x_0$ introduces a multiple-comparisons problem, so simultaneous map-level error control requires external correction.

\newpage
{
\small
\bibliographystyle{abbrvnat}
\bibliography{refs}
}

\newpage
\appendix
\section{Appendix: Proofs}
\label{app:proofs}
\subsection{Proof of Proposition~\ref{prop:conformal_validity}}

\begin{remark}[Validity is agnostic to the score function]
\label{rmk:score_agnostic}
A fundamental property of split conformal prediction \cite{vovk2005, lei2018} is that finite-sample marginal coverage relies on exchangeability of the calibration and test residuals. The theory imposes no regularity, continuity, or boundedness assumptions on the comparison score $S(X,Y)$ itself. Because the centering and scaling functions, $\widehat\theta_r$ and $\widehat\sigma_r$, are fitted on a sample disjoint from the calibration and test sets, conditional exchangeability of the residuals is preserved. Consequently, the choice of score function and the choice of local centering or scaling affect statistical efficiency, such as selection rate and power, but not the finite-sample validity of the conformal bound.
\end{remark}

\begin{proof}[Proof of Proposition~\ref{prop:conformal_validity}]
Condition on the data in the fitting and estimation splits, denoted collectively by $\mathcal{D}_{0}$. After conditioning on $\mathcal{D}_{0}$, the fitted predictors, the localized center $\widehat\theta_r(\cdot)$, and the localized scale $\widehat\sigma_r(\cdot)$ are fixed measurable functions. The remaining randomness comes only from the calibration observations $(X_i,Y_i)_{i\in I_{\mathrm{cal}}}$ and the independent test point $(X_{n+1},Y_{n+1})$.

Let
\[
S_{n+1}:=S(X_{n+1},Y_{n+1})
\quad\text{and}\quad
R_{n+1}:=
\frac{S_{n+1}-\widehat\theta_r(X_{n+1})}
{\widehat\sigma_r(X_{n+1})}.
\]
Together with the calibration residuals $R_i$ defined in the main text, the collection
\[
\{R_i:i\in I_{\mathrm{cal}}\}\cup\{R_{n+1}\}
\]
is exchangeable conditional on $\mathcal{D}_{0}$.

Let $n_{\mathrm{cal}}=|I_{\mathrm{cal}}|$ and $k_\alpha=\lceil(n_{\mathrm{cal}}+1)(1-\alpha)\rceil$, as defined in Section~\ref{sec:conformal}. If $k_\alpha>n_{\mathrm{cal}}$, then $\widehat q_{1-\alpha}=\infty$ and the desired coverage statement is immediate. Otherwise, $\widehat q_{1-\alpha}=R_{(k_\alpha)}$, the $k_\alpha$-th order statistic of the calibration residuals. By exchangeability, using random tie-breaking only for the rank argument, the rank of $R_{n+1}$ among the $n_{\mathrm{cal}}+1$ residuals is uniform on $\{1,\ldots,n_{\mathrm{cal}}+1\}$. Moreover, the event $\{R_{n+1}>R_{(k_\alpha)}\}$ implies that this rank is larger than $k_\alpha$. Therefore,
\[
\mathbb{P}\left(R_{n+1}> \widehat q_{1-\alpha}\mid \mathcal{D}_{0}\right)
\le
\frac{n_{\mathrm{cal}}+1-k_\alpha}{n_{\mathrm{cal}}+1}
\le
\alpha,
\]
or equivalently,
\[
\mathbb{P}\left(R_{n+1}\le \widehat q_{1-\alpha}\mid \mathcal{D}_{0}\right)
\ge 1-\alpha.
\]

Because $\widehat\sigma_r(x)>0$ by construction, the event $\{R_{n+1}\le \widehat q_{1-\alpha}\}$ is equivalent to
\[
S_{n+1}
\le
\widehat\theta_r(X_{n+1})
+
\widehat\sigma_r(X_{n+1})\widehat q_{1-\alpha}
=
\widehat U_\alpha(X_{n+1}),
\]
where the last equality follows from the definition of the conformal upper bound. Marginalizing over $\mathcal{D}_{0}$ gives
\[
\mathbb{P}\left(S_{n+1}\le \widehat U_\alpha(X_{n+1})\right)\ge 1-\alpha.
\]

Finally, the false-winner event satisfies
\[
\{\widehat U_\alpha(X_{n+1})<0,\ S_{n+1}\ge 0\}
\subseteq
\{S_{n+1}> \widehat U_\alpha(X_{n+1})\}.
\]
Hence,
\[
\mathbb{P}\left(\widehat U_\alpha(X_{n+1})<0,\ S_{n+1}\ge 0\right)
\le
\mathbb{P}\left(S_{n+1}> \widehat U_\alpha(X_{n+1})\right)
\le \alpha.
\]
\end{proof}

\subsection{Proof of Theorem~\ref{thm:local_winner_consistency}}

\begin{proof}[Proof of Theorem~\ref{thm:local_winner_consistency}]
Throughout the proof, we condition on the fitted predictors produced by the
independent fitting split, so that the evaluation observations are i.i.d. and
\(\mu(x)=\E[S(X,Y)\mid X=x]\) is fixed. Define
\[
N_n\coloneqq \frac{1}{n}\sum_{i=1}^n K_{r_n}(X_i,x_0)S_i,
\qquad
D_n\coloneqq \frac{1}{n}\sum_{i=1}^n K_{r_n}(X_i,x_0).
\]
The ratio \(N_n/D_n\) is identical to the estimator in
\eqref{eq:local-estimator} whenever \(D_n>0\). On the event \(D_n=0\), we may
assign the estimator any fixed value, say zero; the argument below shows that
\(\mathbb{P}(D_n>0)\to1\), so this finite-sample convention has no asymptotic
effect.

Let
\[
\kappa_0\coloneqq \int_{\mathbb{R}^d}K(u)\,du.
\]
By the theorem's kernel condition, \(0<\kappa_0<\infty\). Since \(K\) is
compactly supported, there exists \(R_K<\infty\) such that \(K(u)=0\) whenever
\(\|u\|>R_K\). Since \(p_X\) is continuous at \(x_0\) and \(p_X(x_0)>0\),
\(p_X\) is locally bounded near \(x_0\) and remains positive on a sufficiently
small neighborhood of \(x_0\). Thus, for all large \(n\), the ordinary kernel
localization \(x_0+r_nu\) with \(K(u)\neq0\) stays inside a neighborhood where
the local continuity and moment conditions apply.

First consider the denominator. By the change of variables
\(u=(x-x_0)/r_n\),
\[
\E[D_n]
=
\int_{\mathbb{R}^d}K(u)p_X(x_0+r_nu)\,du.
\]
Because \(p_X\) is continuous at \(x_0\) and \(K\) is bounded with compact
support, dominated convergence gives
\[
\E[D_n]
\longrightarrow
p_X(x_0)\int_{\mathbb{R}^d}K(u)\,du
=
p_X(x_0)\kappa_0
\coloneqq d_0.
\]
Since \(p_X(x_0)>0\) and \(\kappa_0>0\), we have \(d_0>0\).

Next consider the numerator. By iterated expectation,
\[
\E[N_n]
=
\E\!\left[K_{r_n}(X,x_0)\mu(X)\right].
\]
Using the same change of variables,
\[
\E[N_n]
=
\int_{\mathbb{R}^d}
K(u)\mu(x_0+r_nu)p_X(x_0+r_nu)\,du.
\]
The local second-moment condition and Jensen's inequality imply that
\(\mu\) is locally bounded near \(x_0\):
\[
|\mu(x)|
\le
\E[|S(X,Y)|\mid X=x]
\le
\E[S(X,Y)^2\mid X=x]^{1/2}.
\]
Together with the local boundedness of \(p_X\), this gives an integrable
dominating function proportional to \(K(u)\). Since both \(\mu\) and \(p_X\)
are continuous at \(x_0\), dominated convergence yields
\[
\E[N_n]
\longrightarrow
\mu(x_0)p_X(x_0)\int_{\mathbb{R}^d}K(u)\,du
=
\mu(x_0)p_X(x_0)\kappa_0.
\]

It remains to show that \(N_n\) and \(D_n\) concentrate around their
expectations. For the denominator,
\[
\Var(D_n)
=
\frac{1}{n}\Var\!\left(K_{r_n}(X,x_0)\right)
\le
\frac{1}{n}\E\!\left[K_{r_n}(X,x_0)^2\right].
\]
Again changing variables,
\[
\E\!\left[K_{r_n}(X,x_0)^2\right]
=
r_n^{-d}
\int_{\mathbb{R}^d}
K(u)^2p_X(x_0+r_nu)\,du.
\]
Since \(K\) is bounded and compactly supported, \(K^2\) is integrable; since
\(p_X\) is locally bounded at \(x_0\), the integral is \(O(1)\). Hence
\[
\Var(D_n)
=
O\!\left(\frac{1}{nr_n^d}\right)
\longrightarrow0,
\]
because \(nr_n^d\to\infty\). Therefore
\[
D_n-\E[D_n]\overset{p}{\longrightarrow}0.
\]

For the numerator,
\[
\Var(N_n)
=
\frac{1}{n}\Var\!\left(K_{r_n}(X,x_0)S\right)
\le
\frac{1}{n}\E\!\left[K_{r_n}(X,x_0)^2S^2\right].
\]
By iterated expectation and the same change of variables,
\[
\E\!\left[K_{r_n}(X,x_0)^2S^2\right]
=
r_n^{-d}
\int_{\mathbb{R}^d}
K(u)^2
\E\!\left[S^2\mid X=x_0+r_nu\right]
p_X(x_0+r_nu)\,du.
\]
For all large \(n\), the compact support of \(K\) keeps \(x_0+r_nu\) inside
the local neighborhood where \(\E[S^2\mid X=x]\) is bounded and where \(p_X\)
is bounded. Therefore
\[
\E\!\left[K_{r_n}(X,x_0)^2S^2\right]
=
O(r_n^{-d}),
\]
and hence
\[
\Var(N_n)
=
O\!\left(\frac{1}{nr_n^d}\right)
\longrightarrow0.
\]
Thus
\[
N_n-\E[N_n]\overset{p}{\longrightarrow}0.
\]

Combining the expectation limits and concentration bounds gives
\[
D_n\overset{p}{\longrightarrow}p_X(x_0)\kappa_0
\]
and
\[
N_n\overset{p}{\longrightarrow}\mu(x_0)p_X(x_0)\kappa_0.
\]
Because \(p_X(x_0)\kappa_0>0\), we also have
\[
\mathbb{P}(D_n>0)\longrightarrow1.
\]
Consequently, by Slutsky's theorem,
\[
\hat\mu_n(x_0)
=
\frac{N_n}{D_n}
\overset{p}{\longrightarrow}
\mu(x_0).
\]

It remains to translate consistency of the local mean estimator into
consistency of the local winner. If \(\mu(x_0)<0\), then
\(\delta^\star(x_0)=A\), and
\[
\mathbb{P}\!\left(\hat\delta_n(x_0)\neq A\right)
=
\mathbb{P}\!\left(\hat\mu_n(x_0)\ge0\right).
\]
Moreover,
\[
\{\hat\mu_n(x_0)\ge0\}
\subseteq
\left\{
|\hat\mu_n(x_0)-\mu(x_0)|
\ge
|\mu(x_0)|
\right\}.
\]
Since \(\hat\mu_n(x_0)\overset{p}{\to}\mu(x_0)\), the probability of the event
on the right converges to zero. Hence
\[
\mathbb{P}\!\left(\hat\delta_n(x_0)=A\right)\longrightarrow1.
\]
The case \(\mu(x_0)>0\) is identical and gives
\[
\mathbb{P}\!\left(\hat\delta_n(x_0)=B\right)\longrightarrow1.
\]
Therefore, whenever \(\mu(x_0)\neq0\),
\[
\mathbb{P}\!\left(
\hat\delta_n(x_0)=\delta^\star(x_0)
\right)
\longrightarrow1.
\]

The same proof applies to the estimation-split estimator in Section~4 after
replacing \(n\) by \(|I_{\mathrm{est}}|\), with the corresponding radius
sequence satisfying the same shrinking-neighborhood and diverging-effective-
sample-size conditions.
\end{proof}

\subsection{Proof of Proposition~\ref{prop:global_local_mismatch}}

\begin{proof}[Proof of Proposition~\ref{prop:global_local_mismatch}]
We first prove consistency of the aggregate selector for the sign of the global mean score. Let $S_i:=S(X_i,Y_i)$. By definition,
\[
\bar S_n=\frac{1}{n}\sum_{i=1}^n S_i.
\]
Since $\mathbb{E}[|S(X,Y)|]<\infty$, the weak law of large numbers implies
\[
\bar S_n \overset{p}{\longrightarrow} \mathbb{E}[S(X,Y)]=\mu_{\mathrm{glob}}.
\]

Suppose first that $\mu_{\mathrm{glob}}>0$. Then the true aggregate winner is $B$, and
\[
\mathbb{P}\left(\hat\delta_n^{\mathrm{agg}}=B\right)
=
\mathbb{P}(\bar S_n\ge 0).
\]
Hence
\[
\mathbb{P}(\bar S_n<0)
\le
\mathbb{P}\left(|\bar S_n-\mu_{\mathrm{glob}}|\ge \mu_{\mathrm{glob}}\right)
\longrightarrow 0,
\]
so
\[
\mathbb{P}\left(\hat\delta_n^{\mathrm{agg}}=B\right)\longrightarrow 1.
\]
The case $\mu_{\mathrm{glob}}<0$ is identical and yields
\[
\mathbb{P}\left(\hat\delta_n^{\mathrm{agg}}=A\right)\longrightarrow 1.
\]
Therefore,
\[
\mathbb{P}\left(\hat\delta_n^{\mathrm{agg}}=\delta^{\mathrm{agg},\star}\right)\longrightarrow 1.
\]

We next prove the existence statement. Fix any $q\in(1/2,1)$. Let $X\sim \mathrm{Unif}(0,1)$, and let $\mathcal R_A=[0,q]$. Define
\[
\mu(x)=
\begin{cases}
-a, & x\in \mathcal R_A,\\
b, & x\notin \mathcal R_A,
\end{cases}
\]
where $a>0$ and $b>qa/(1-q)$. Then
\[
\mathbb{P}(\mu(X)<0)=\mathbb{P}(X\in \mathcal R_A)=q.
\]
Moreover,
\[
\mathbb{E}[\mu(X)]
=
-q\,a+(1-q)b
>
0.
\]
Let $\eta$ be any mean-zero noise, independent of $X$, with $\mathbb{E}[|\eta|]<\infty$ and $\mathrm{Var}(\eta)>0$, and define
\[
S:=\mu(X)+\eta.
\]
Then $\mathbb{E}[|S|]<\infty$, $\mathbb{E}[S\mid X=x]=\mu(x)$, and the score is non-degenerate because $\mathrm{Var}(S\mid X=x)=\mathrm{Var}(\eta)>0$. Therefore
\[
\mathbb{P}(\mu(X)<0)=q
\qquad\text{and}\qquad
\mathbb{E}[S]>0.
\]
Applying the aggregate selector to i.i.d. draws from this constructed distribution, the first part of the proof gives
\[
\mathbb{P}\left(\hat\delta_n^{\mathrm{agg}}=B\right)\longrightarrow 1.
\]
This proves that aggregate selection can asymptotically favor model $B$ even though model $A$ is locally superior on at least a $q$-fraction of the covariate distribution.
\end{proof}

\subsection{Proof of Proposition~\ref{prop:squared_loss_decomp}}

\begin{proof}[Proof of Proposition~\ref{prop:squared_loss_decomp}]
Fix $x\in\cX$. Conditional on $X=x$, write $Y=f^\star(x)+\varepsilon$ with $\E[\varepsilon\mid X=x]=0$. For a generic model $m\in\{A,B\}$,
\begin{align}
L_m(X,Y)\mid X=x
&=\bigl(\hat f_m(x)-f^\star(x)-\varepsilon\bigr)^2 \nonumber\\
&=\bigl(\hat f_m(x)-f^\star(x)\bigr)^2 - 2\varepsilon\bigl(\hat f_m(x)-f^\star(x)\bigr) + \varepsilon^2.
\label{eq:sq-expand-app}
\end{align}

Taking expectation over both the training randomness and the fresh test noise, we isolate the cross-term. Because the fresh test noise $\varepsilon$ is independent of the fitting sample $\mathcal{D}$ conditional on $X$, we have $\mathbb{E}[\varepsilon \mid X=x, \mathcal{D}] = \mathbb{E}[\varepsilon \mid X=x] = 0$. Therefore,
\begin{align}
\E\bigl[\varepsilon \hat f_m(x)\mid X=x\bigr]
&= \E\Bigl[\E\bigl[\varepsilon \hat f_m(x)\mid X=x,\mathcal{D}\bigr]\mid X=x\Bigr] \nonumber\\
&= \E\Bigl[\hat f_m(x)\,\E\bigl[\varepsilon\mid X=x,\mathcal{D}\bigr]\mid X=x\Bigr] \nonumber\\
&= 0.
\label{eq:cross-term-vanish-app}
\end{align}

Consequently,
\begin{equation}
\E[L_m(X,Y)\mid X=x] = \E\bigl[(\hat f_m(x)-f^\star(x))^2\bigr]+\E[\varepsilon^2\mid X=x].
\label{eq:loss-noise-split-app}
\end{equation}

Applying the standard bias--variance identity $\E[(Z-a)^2]=\Var(Z)+(\E[Z]-a)^2$ with $Z=\hat f_m(x)$ and $a=f^\star(x)$ gives
\begin{equation}
\E[L_m(X,Y)\mid X=x] = \mathrm{Var}_m(x)+\mathrm{Bias}_m^2(x)+\E[\varepsilon^2\mid X=x].
\label{eq:model-risk-decomp-app}
\end{equation}

Subtracting the identities for $m=A$ and $m=B$ cancels the irreducible noise term $\E[\varepsilon^2\mid X=x]$, yielding
\[
\E[S(X,Y)\mid X=x] = \mathrm{Bias}_A^2(x)-\mathrm{Bias}_B^2(x)+\mathrm{Var}_A(x)-\mathrm{Var}_B(x).
\]
Multiplying both sides by $K_r(X,x_0)$, taking expectation, and dividing by $\E[K_r(X,x_0)]$ yields \eqref{eq:local-bias-var-main}.
\end{proof}

\section{Appendix: Additional Experimental Details}
\label{app:additional_exp}

This appendix provides full details for the experiments reported in Section~\ref{sec:experiments}. We include the complete data-generating processes, visual illustrations of the fitted models, detailed non-conformal winner-identification metrics, and real datasets summary.

\subsection{Synthetic data-generating processes}
\label{app:DGP}
Table~\ref{tab:dgps_app} reports the full synthetic data-generating processes. Figure~\ref{fig:dgp_app} visualizes the five designs, together with representative fitted models. Case~1 is a baseline where $A$ is better almost everywhere. Cases~2 and~3 create global-local mismatch. In Case~3 (Figure~\ref{fig:biasvar}), a narrow nonlinear island is strong enough to make the global average favor $A$. Specifically, A wins inside [0.42,0.58] and B wins about 84\% of the region. Case~4 increases the noise level while keeping the Case~1 signal, creating a stress test for false declarations. Case~5 introduces strong heteroskedasticity, testing whether local scaling improves conformal efficiency. 

\begin{table}[h]
\centering
\scriptsize
\caption{Synthetic data-generating processes. All cases use $X\sim\mathrm{Unif}[0,1]$ and $Y=f^\star(X)+\varepsilon$.}
\label{tab:dgps_app}
\begin{tabularx}{\textwidth}{@{}c l X p{2.8cm} p{2.2cm}@{}}
\toprule
ID & Name & $f^\star(x)$ / noise & Oracle pattern & Main role \\
\midrule
1 & \texttt{homog\_a}        & $2(x-0.5)+0.75\,\mathrm{sign}(\sin(6\pi x))$;\ $\sigma=0.1$ & $A$ wins nearly everywhere & power baseline \\
2 & \texttt{nonlinear\_tail} & $2x$ for $x\le 0.9$; $2x+0.30\sin(2\pi(x-0.9)/0.1)$ otherwise; $\sigma=0.30$ & $B$ wins globally; $A$ wins in the right tail & global/local mismatch \\
3 & \texttt{middle\_island}  & $2x$ outside $[0.42,0.58]$; $2x+1.20\sin(2\pi(x-0.42)/0.16)$ inside; $\sigma=0.30$ & $A$ wins globally through a narrow island; $B$ wins on most $x$ & map recovery; bandwidth sensitivity \\
4 & \texttt{validity\_stress}& Same signal as Case~1; $\sigma=0.40$ & $A$ wins globally; realized scores are noisy & CP necessity \\
5 & \texttt{hetero\_scale}   & Same signal as Case~1; $\sigma(x)=0.1$ for $x\le 0.5$ and $2.0$ otherwise & $A$ wins globally; score dispersion changes sharply & local scaling \\
\bottomrule
\end{tabularx}
\end{table}

\begin{figure}[h]
\centering
\includegraphics[width=\linewidth]{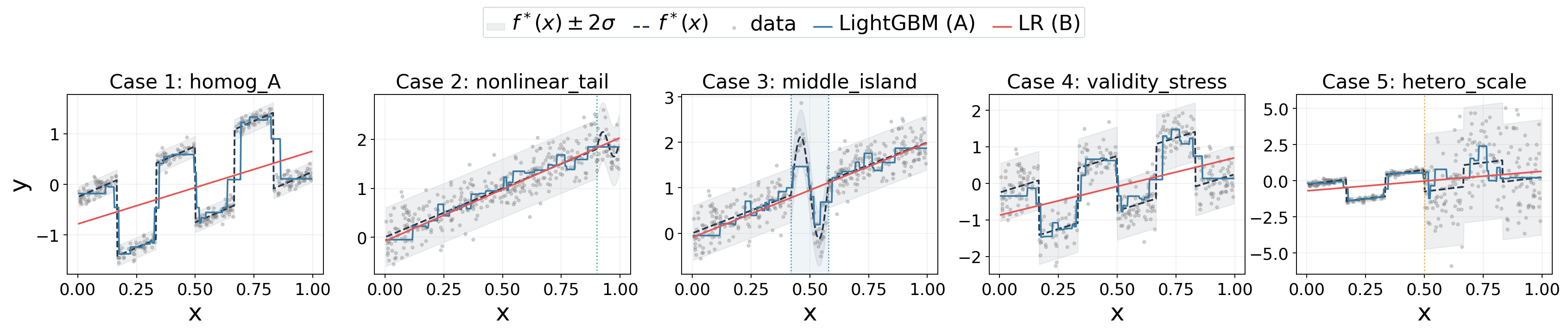}
\caption{Synthetic DGPs. Black dashed: $f^\star(x)$; gray band: $f^\star(x)\pm 2\sigma(x)$; blue: a representative LightGBM ($A$) fit; red: linear ($B$) fit.}
\label{fig:dgp_app}
\end{figure}

\subsection{Evaluation metrics for synthetic experiments}
\label{app:metrics}

Table~\ref{tab:metrics} defines all evaluation metrics used in synthetic experiments, separated into non-conformal winner metrics that assess oracle map recovery and conformal winner metrics that assess the reliability, selectivity, and realized gain of declared $A$ regions.

\begin{table}[h]
\centering
\footnotesize
\caption{Evaluation metrics for synthetic experiments. Non-conformal winner metrics assess whether the estimated map matches the oracle local winner regions. Conformal winner metrics assess whether declared $A$ regions are reliable, selective, and useful.}
\label{tab:metrics}
\begin{tabularx}{\textwidth}{@{}p{2.5cm} p{2.0cm} X p{4.0cm}@{}}
\toprule
Role & Metric & Definition & Interpretation \\
\midrule
\multirow{3}{2.6cm}{Non-conformal winner metrics} & Winner accuracy & $\Pbb(\hat w_\alpha(X)=w^\star(X)\mid|\mu(X)|>\tau)$ & Overall agreement with the oracle winner map outside the tie band \\
& Region-IoU$_A$ & IoU between predicted and oracle $A$ regions & Spatial overlap with the oracle $A$-winning region \\
& Region-IoU$_B$ & IoU between predicted and oracle $B$ regions & Spatial overlap with the oracle $B$-winning region \\
\midrule
\multirow{4}{2.6cm}{Conformal winner metrics} & Selection rate & $\Pbb(\hat w_\alpha=A)$ & Fraction of test points declared as $A$ wins \\
& False winner rate & $\Pbb(\hat w_\alpha=A,\,S\ge 0)$ & Empirical false-winner rate for $A$ declarations \\
& Power & $\Pbb(\hat w_\alpha=A\mid\mu(X)<-\tau)$ & Coverage of the true oracle $A$ region \\
& Conditional gain & $\E[L_B-L_A\mid\hat w_\alpha=A]=-\E[S\mid\hat w_\alpha=A]$ & Realized value of the selected $A$ declarations \\
\bottomrule
\end{tabularx}
\end{table}

\subsection{Non-conformal winner-identification results}

Table~\ref{tab:winner_metrics_app} reports local-winner identification results for the forced non-conformal winner maps. These results complement the main text by showing full accuracy and region-overlap values across all five synthetic cases. The table is not used to establish finite-sample validity; conformal validity is evaluated separately through false-winner rates in the main paper.

\begin{table}[h]
\centering
\scriptsize
\caption{Local-winner identification metrics at $\alpha=0.10$ (mean over $10$ seeds). Acc: winner accuracy outside the tie band; IoU$_A$ and IoU$_B$: intersection-over-union with the oracle $A$ and $B$ regions. All rows are forced winner maps and do not carry conformal guarantees.}
\label{tab:winner_metrics_app}
\begin{tabular}{@{}l l c c c@{}}
\toprule
Case & Method & Acc & IoU$_A$ & IoU$_B$ \\
\midrule
\multirow{4}{*}{1: \texttt{homog\_a}}
& \texttt{global\_mean}      & 0.992 & 0.992 & 0.000 \\
& \texttt{local\_def1}       & 0.992 & 0.992 & 0.000 \\
& \texttt{local\_def2}       & 0.992 & 0.992 & 0.000 \\
& \texttt{local\_def3}       & 0.988 & 0.988 & 0.056 \\
\midrule
\multirow{4}{*}{2: \texttt{nonlinear\_tail}}
& \texttt{global\_mean}      & 0.877 & 0.000 & 0.820 \\
& \texttt{local\_def1}       & 0.766 & 0.268 & 0.676 \\
& \texttt{local\_def2}       & 0.714 & 0.238 & 0.628 \\
& \texttt{local\_def3}       & 0.881 & 0.014 & 0.822 \\
\midrule
\multirow{4}{*}{3: \texttt{middle\_island}}
& \texttt{global\_mean}      & 0.293 & 0.316 & 0.000 \\
& \texttt{local\_def1}       & 0.823 & 0.418 & 0.623 \\
& \texttt{local\_def2}       & 0.783 & 0.398 & 0.582 \\
& \texttt{local\_def3}       & 0.908 & 0.340 & 0.762 \\
\midrule
\multirow{4}{*}{4: \texttt{validity\_stress}}
& \texttt{global\_mean}      & 0.992 & 0.992 & 0.000 \\
& \texttt{local\_def1}       & 0.986 & 0.986 & 0.000 \\
& \texttt{local\_def2}       & 0.985 & 0.985 & 0.000 \\
& \texttt{local\_def3}       & 0.515 & 0.512 & 0.012 \\
\midrule
\multirow{4}{*}{5: \texttt{hetero\_scale}}
& \texttt{global\_mean}      & 0.855 & 0.851 & 0.000 \\
& \texttt{local\_def1}       & 0.816 & 0.791 & 0.328 \\
& \texttt{local\_def2}       & 0.813 & 0.790 & 0.320 \\
& \texttt{local\_def3}       & 0.650 & 0.584 & 0.290 \\
\bottomrule
\end{tabular}
\end{table}

\subsection{Real datasets summary}

Table~\ref{tab:realdata_summary} reports the size, dimensionality, and global comparison score distribution for the four real-data benchmarks used in Section~\ref{subsec:real_exp_design}.

\begin{table}[h]
\centering
\small
\caption{Real datasets summary. $\bar S$ is the global mean comparison score on the calibration split, reported as mean\,$\pm$\,std across $10$ seeds. $S<0$ favors model~$A$ (MLP).}
\label{tab:realdata_summary}
\begin{tabular}{@{}lrrr@{}}
\toprule
Dataset & $n$ & $d$ & $\bar S$ \\
\midrule
\texttt{concrete}    & 1{,}030  &   8 & $-0.6 \pm 3.1$  \\
\texttt{auto\_mpg}   &   398    &   7 & $-0.5 \pm 3.5$  \\
\texttt{blog\_data}  & 52{,}397 & 280 & $-56.9 \pm 63.4$  \\
\texttt{facebook\_1} & 40{,}948 &  53 & $-115.7 \pm 133.5$ \\
\bottomrule
\end{tabular}
\end{table}
\section{Appendix: Ablation Studies}
\label{subsec:exp_ablation}

\paragraph{Score family.}
Figure~\ref{fig:score_family} replaces \texttt{gap\_mse} with the standardized gap $s_{\mathrm{std}}(L_A,L_B;x)=(L_A-L_B)/\hat v(x)$ and the log-ratio $T=\log((L_A+\tau)/(L_B+\tau))$. The conformal validity guarantee is score-agnostic for the local CP construction, but selection and power vary. \texttt{gap\_mse} is the only score with nonzero power on all five cases and is strongest on the two boundary cases and the validity-stress case. The log-ratio and standardized gap improve power on the homogeneous and heteroskedastic cases, but they lose the nonlinear-tail signal and are weaker on Case~3. We therefore keep \texttt{gap\_mse} as the primary score and treat the alternatives as useful but case-dependent sensitivity checks.

\begin{figure}[t]
\centering
\includegraphics[width=\linewidth]{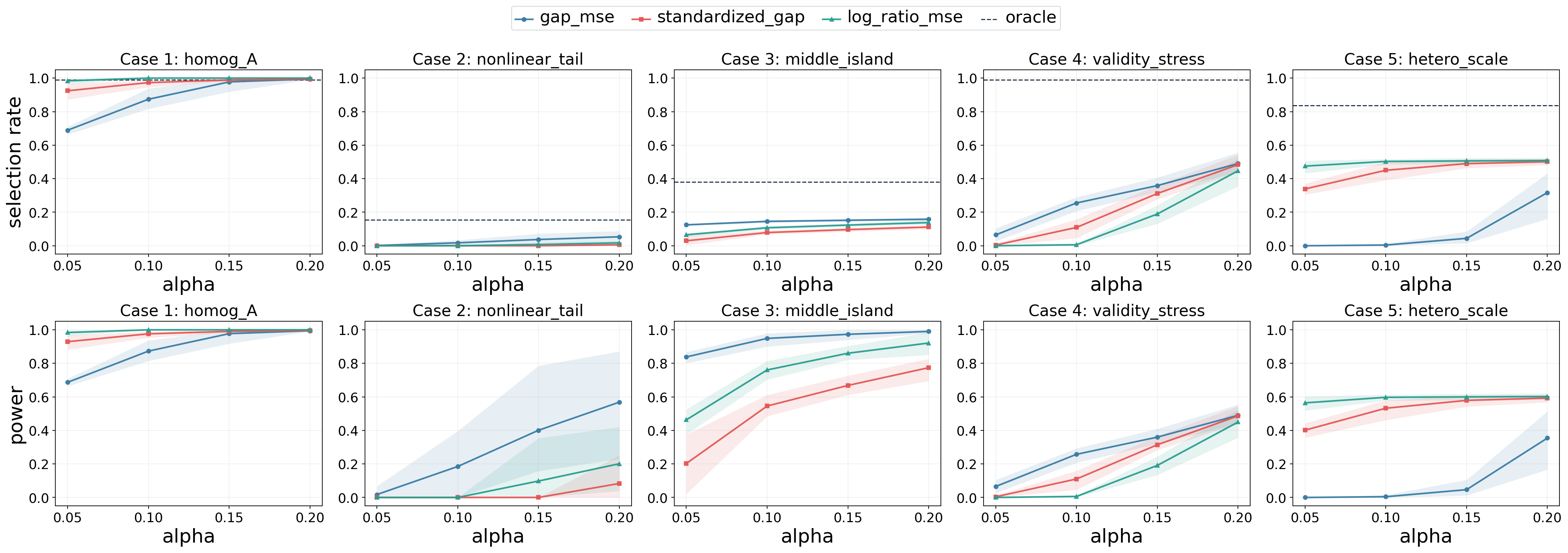}
\caption{Score-family ablation. Selection rate (top) and oracle $A$ power (bottom) for \texttt{local\_split\_cp} under the three scores \texttt{gap\_mse}, \texttt{standardized\_gap}, and \texttt{log\_ratio\_mse}, all five cases.}
\label{fig:score_family}
\end{figure}

\paragraph{Locality parameter $k$.}
Figure~\ref{fig:k_sweep} sweeps $k\in\{5,10,20,50,100,200,400\}$ and plots selection rate (top) and oracle $A$ power (bottom) for \texttt{local\_split\_cp} on all five cases at $\alpha=0.10$. The homogeneous Case~1 is essentially flat in $k$ as expected. Cases with local structure split, such as Case~2, loses power for large $k$ because the tail signal is averaged away.

\begin{figure}[t]
\centering
\includegraphics[width=\linewidth]{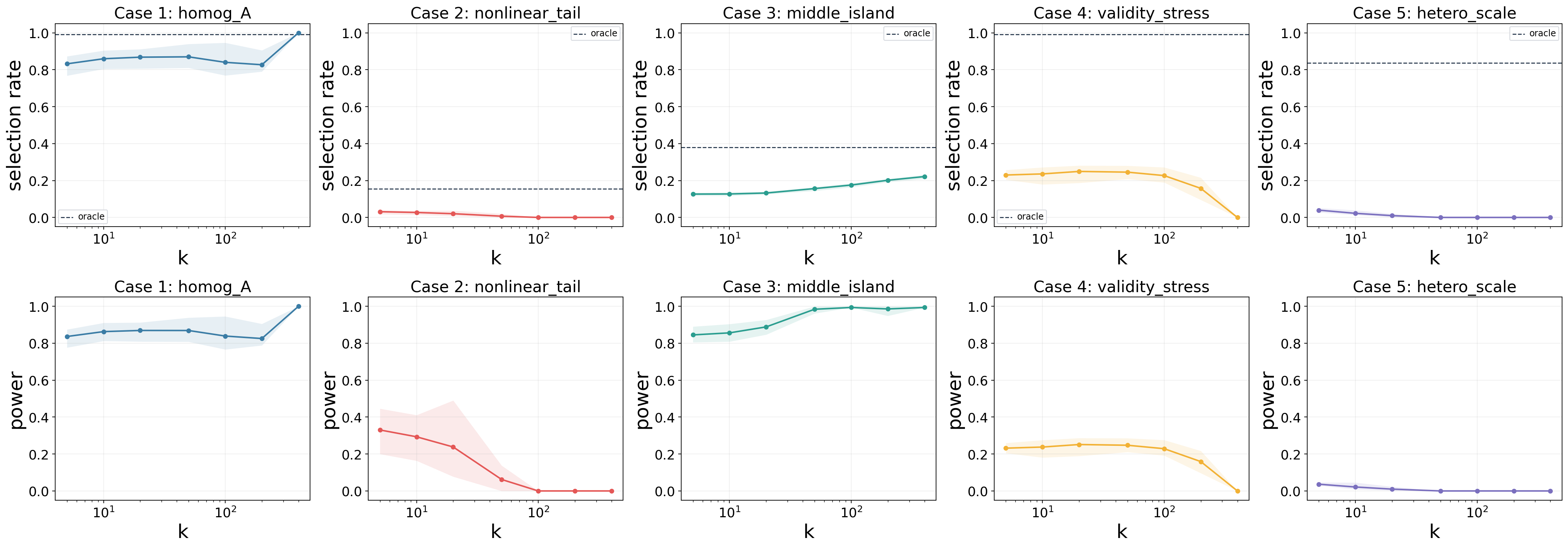}
\caption{Locality parameter ablation. Selection rate (top) and oracle $A$ power (bottom) for \texttt{local\_split\_cp} as $k$ varies, all five cases at $\alpha=0.10$.}
\label{fig:k_sweep}
\end{figure}

\paragraph{Covariate dimension.}
Figure~\ref{fig:highdim} keeps the signal in the first coordinate of $X$ and pads the remaining $d-1$ coordinates with i.i.d.\ uniform nuisance covariates, sweeping $d\in\{1,10,30,100\}$. Cases with local boundary structure (Cases~2 and~3) lose power with $d$ increases, consistent with the curse of dimensionality for $k$-NN. Case~1 is much less sensitive at moderate $d$ because there is no boundary to recover. In a real high-dimensional application, the localizer should be applied to a learned low-dimensional representation rather than the raw covariates, as discussed in Section~4.

\begin{figure}[t]
\centering
\includegraphics[width=\linewidth]{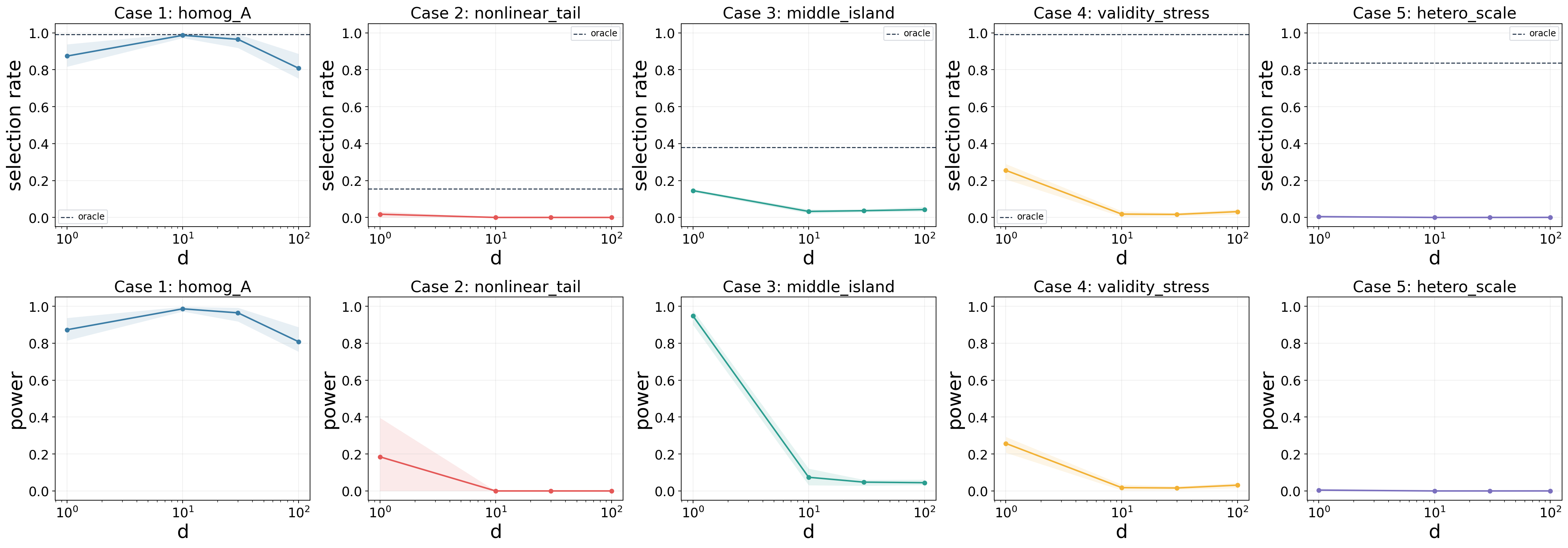}
\caption{Covariate-dimension ablation. horizontal axis is the ambient dimension $d$, with the signal kept in the first coordinate and the remaining $d-1$ coordinates filled with uniform nuisance.}
\label{fig:highdim}
\end{figure}

\section{Appendix: Compute Resources}
\label{app:compute_resources}

All experiments were run on a single MacBook Pro with an Apple M3 chip (8 physical cores), 16\,GB unified memory, and macOS~14.6 (Sonoma). No GPU was used; all computation relied on the CPU. The software environment used Python~3.12.7 with NumPy~1.26.4, pandas~2.2.3, scikit-learn~1.6.1, LightGBM~4.5.0, matplotlib~3.7.5, and SciPy~1.15.2.

For the synthetic experiments in Section~\ref{subsec:exp_design}, a full run over all 10 random seeds took approximately 10 minutes. For the real-data experiments in Section~\ref{subsec:real_exp_design}, a full run over all benchmark datasets took a similar amount of time. Across all experiments reported in the paper, the total compute budget was well under 1 CPU-hour. The experiments consist of small-scale regression tasks and do not require large-scale distributed training.

We used 10 random seeds (0--9) for all experiments. No pretrained models were used. For the synthetic experiments, model $A$ is LightGBM and model $B$ is linear regression, both trained from scratch on each data split using the hyperparameters described in Section~\ref{subsec:exp_design}. For the real-data experiments, model $A$ is a small MLP (scikit-learn \texttt{MLPRegressor} with dataset-specific hidden layer sizes, \texttt{max\_iter=1000}, and early stopping) and model $B$ is linear regression; both are trained from scratch on each split as described in Section~\ref{subsec:real_exp_design}. No deep learning framework beyond scikit-learn was required.

\clearpage

\end{document}